\def\year{2022}\relax
\documentclass[letterpaper]{article} 
\usepackage{aaai22}  
\usepackage{times}  
\usepackage{helvet}  
\usepackage{courier}  
\usepackage[hyphens]{url}  
\usepackage{graphicx} 
\usepackage{natbib}  
\usepackage{caption} 
\DeclareCaptionStyle{ruled}{labelfont=normalfont,labelsep=colon,strut=off} 
\usepackage{algorithm}
\usepackage[algo2e]{algorithm2e}
\usepackage{amsmath}
\usepackage{amssymb}
\usepackage{amsthm}
\usepackage{nicefrac}
\usepackage{booktabs}
\usepackage{multirow}
\usepackage{color}
\usepackage{mathtools}
\usepackage{newfloat}
\usepackage{listings}
\usepackage{xcolor}
\usepackage{color, colortbl}
\definecolor{Gray}{gray}{.9}
\newcolumntype{g}{>{\columncolor{Gray}}c}
\floatstyle{ruled}
\newfloat{listing}{tb}{lst}{}
\floatname{listing}{Listing}

\definecolor{codegreen}{rgb}{0,0.6,0}
\definecolor{codegray}{rgb}{0.5,0.5,0.5}
\definecolor{codepurple}{rgb}{0.58,0,0.82}
\definecolor{backcolour}{rgb}{0.95,0.95,0.92}

\lstdefinestyle{mystyle}{
    backgroundcolor=\color{backcolour},   
    commentstyle=\color{codegreen},
    keywordstyle=\color{magenta},
    numberstyle=\tiny\color{codegray},
    stringstyle=\color{codepurple},
    basicstyle=\ttfamily\footnotesize,
    breakatwhitespace=false,         
    breaklines=true,                 
    captionpos=b,                    
    keepspaces=true,                 
    numbers=left,                    
    numbersep=5pt,                  
    showspaces=false,                
    showstringspaces=false,
    showtabs=false,                  
    tabsize=2
}

\DeclarePairedDelimiter{\ceil}{\lceil}{\rceil}
\title{Combating Adversaries with Anti-Adversaries}
\author {
    Motasem Alfarra\textsuperscript{\rm 1},
    Juan C. P\'erez\textsuperscript{\rm 1},
    Ali Thabet\textsuperscript{\rm 2}, 
    Adel Bibi\textsuperscript{\rm 3}, \\
    Philip H.S. Torr\textsuperscript{\rm 3}, and
    Bernard Ghanem\textsuperscript{\rm 1}
}
\affiliations {
    \textsuperscript{\rm 1} King Abdullah University of Science and Technology (KAUST), 
    \textsuperscript{\rm 2} Facebook Reality Labs,
    \textsuperscript{\rm 3} University of Oxford \\
    motasem.alfarra@kaust.edu.sa
}

\usepackage{bibentry}
\newtheorem{theorem}{Theorem}
\newtheorem{proposition}{Proposition}
\newtheorem{corollary}{Corollary}
\newtheorem{definition}{Definition}

\newtheorem{theo}{Theorem}
\newtheorem{prop}{Proposition}
\newtheorem{cor}{Corollary}

\newcommand{\argmax}{\text{arg}\max}
\newcommand{\argmin}{\text{arg}\min}
\newcommand{\eg}{\emph{e.g.~}}
\newcommand{\ie}{\emph{i.e.~}}
\newcommand{\etc}{\emph{etc.}}
\begin{document}

\maketitle

\begin{abstract}
    Deep neural networks are vulnerable to small input perturbations known as adversarial attacks. Inspired by the fact that these adversaries are constructed by iteratively minimizing the confidence of a network for the true class label, we propose the anti-adversary layer, aimed at countering this effect. In particular, our layer generates an input perturbation in the opposite direction of the adversarial one and feeds the classifier a perturbed version of the input. Our approach is training-free and theoretically supported. We verify the effectiveness of our approach by combining our layer with both nominally and robustly trained models and conduct large-scale experiments from black-box to adaptive attacks on CIFAR10, CIFAR100, and ImageNet. Our layer significantly enhances model robustness while coming at no cost on clean accuracy.\footnote{Official code: https://github.com/MotasemAlfarra/Combating-Adversaries-with-Anti-Adversaries}
\end{abstract}

\section{Introduction}\label{introduction}
Deep Neural Networks (DNNs) are vulnerable to small input perturbations known as adversarial attacks \cite{szegedy2013intriguing, goodfellow2014explaining}. In particular, a classifier $f$, which correctly classifies $x$, can be fooled by a small adversarial perturbation $\delta$ into misclassifying $(x+\delta)$ even though $x$ and $(x+\delta)$ are indistinguishable to the human eye. Such perturbations can compromise trust in DNNs, hindering their use in safety- and security-critical applications, \eg self-driving cars \cite{sitawarin2018darts}. 
\begin{figure}
    \centering
    \includegraphics[width=0.325\textwidth]{ 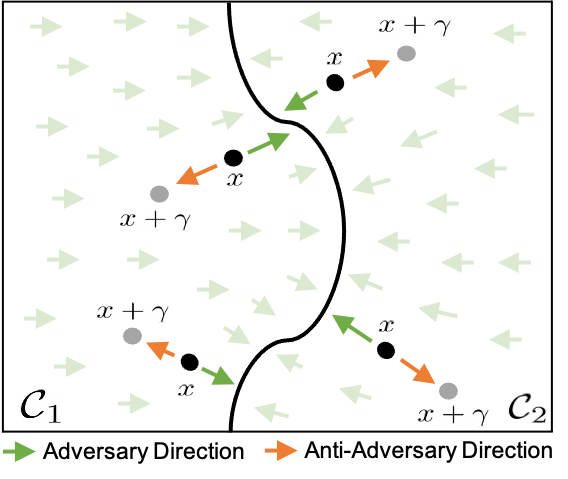}
    \caption{\textbf{Anti-adversary classifier.} The flow field of adversarial perturbations is shown in light green for both classes $\mathcal{C}_1$ and $\mathcal{C}_2$. The anti-adversary we construct pulls a given point $x$ to $(x+\gamma)$ by moving in the direction \textit{opposite} to that of the adversary flow field (orange arrows). 
    }
    \label{fig:pull}
\end{figure}
While there have been extensive efforts aimed at training DNNs that are robust to adversarial attacks, assessing the robustness of defenses remains an elusive task. This difficulty is due to the following reasons. (\textbf{i}) The robustness of models varies according to the information an attacker is assumed to know, \eg training data, gradients, logits, \etc, which, for ease, dichotomously categorizes adversaries as being black- or white-box. Consequently, this categorization results in difficulties when comparing defenses tailored to a specific type of adversaries. For instance, several defenses crafted for robustness against white-box adversaries were later broken by their weaker black-box counterparts~\cite{practical, brendel2018decision}. (\textbf{ii}) In addition, empirically-evaluated robustness can be overestimated if fewer efforts are invested into \textit{adaptively} constructing a stronger attack~\cite{adaptive, carlini2019evaluating}. The lack of reliable assessments has been responsible for a false sense of security, as several thought-to-be-strong defenses against white-box adversaries were later broken with better carefully-crafted adaptive attacks~\cite{athalye2018obfuscated}. The few defenses that have stood the test of time usually come at the expense of costly training and performance degradation on clean samples \cite{tsipras2018robustness}. Even worse, while most of these defenses are meant to resist white-box attacks, little effort has been invested into resisting the black-box counterparts, which may be more common and practical~\cite{snd}, as online APIs such as IBM Watson and Azure tend to abstain from disclosing information about the inner workings of their models.

In this work, we propose a simple, generic, and training-free layer that improves the robustness of both nominally- and robustly-trained models. Specifically, given a base classifier $f : \mathbb{R}^n \rightarrow \mathcal{Y}$, which maps $\mathbb{R}^n$ to labels in the set $\mathcal{Y}$, and an input $x$, our layer constructs a data- and model-dependent perturbation $\gamma$ in the \textit{anti-adversary} direction, \ie the direction that maximizes the base classifier's confidence on the pseudo-label $f(x)$, as illustrated in Figure \ref{fig:pull}. The new sample $(x+\gamma)$ is then fed to the base classifier $f$ in lieu of $x$. We dub this complete approach as the \textit{anti-adversary} classifier $g$. By conducting an extensive robustness assessment of our classifier $g$ on several datasets and under the full spectrum of attacks, from black-box --arguably the most realistic-- and white-box, to adaptive attacks, we find across-the-board improvements in robustness over all base classifiers $f$.

\noindent \textbf{Contributions.} (\textbf{i}) We propose an anti-adversary layer to improve the adversarial robustness of base classifiers. Our proposed layer comes at marginal computational overhead and virtually no impact on clean accuracy. Moreover, we provide theoretical insights into the robustness enhancement that our layer delivers. (\textbf{ii}) We demonstrate empirically under black-box attacks that our layer positively interacts with both nominally trained and state-of-the-art robust models, \eg TRADES~\cite{trades}, ImageNet-Pre \cite{pretraining}, MART \cite{mart}, HYDRA \cite{hydra}, and AWP \cite{awp}, on CIFAR10, CIFAR100 \cite{cifars} and ImageNet \cite{krizhevsky2012imagenet}. Our results show that the anti-adversary layer not only improves robustness against a variety of black-box attacks \cite{ilyas2018prior, ilyas2018black, andriushchenko2020square}, but also that this improvement comes at no cost on clean accuracy and does not require retraining. (\textbf{iii}) We further evaluate our approach on a challenging setting, in which the attacker is granted full access to the anti-adversary classifier, \ie white-box attacks. Under this setup, we equip the five aforementioned defenses with our classifier and test them under the strong AutoAttack benchmark \cite{autoattack}. Our experiments report across-the-board average improvements of 19\% and 11\% on CIFAR10 and CIFAR100, respectively.

\section{Related Work}
\textbf{Adversarial Attacks.} Evaluating network robustness dates back to the works of \cite{szegedy2013intriguing, goodfellow2014explaining}, where it was shown that small input perturbations, dubbed as adversarial attacks, can change network predictions. Follow-up methods present a variety of ways to construct such attacks, which are generally categorized as black-box, white-box and adaptive attacks. Black-box attackers either carry out zeroth order optimization to maximize a suitably defined loss function \cite{simba, pmlr-v80-uesato18a}, or learn offline adversaries that transfer well across networks \cite{transfer_attack, Bhagoji_2018_ECCV}. On the other hand, and less practical, white-box attackers are assumed to have the full knowledge of the network, \eg parameters, gradients, architecture, and training data, among others \cite{moosavi2016deepfool, madry2017towards}. Despite that, previously proposed attackers from this family often construct adversaries solely based on network predictions and gradients with respect to the input \cite{carlini2017towards, croce2020minimally}. Although this results in an overestimation of the worst-case robustness for networks, it has now become the \textit{de facto} standard for benchmarking robustness \cite{autoattack}. It was recently demonstrated that several networks, which were shown to be robust in the white-box setting, were susceptible to weaker black-box attacks \cite{dong2020benchmarking}. Consequently, there has been significant interest for \textit{adaptive attacks}, \ie specifically tailored adversaries exploiting complete knowledge of the network (not only predictions and gradients), for a reliable worst-case robustness assessment \cite{adaptive, athalye2018obfuscated}. However, while worst-case robustness is of interest through adaptive attacks, it may not be of practical relevance. We argue that a proper robustness evaluation should cover the full spectrum of attackers from black-box to adaptive attacks; thus, in this paper, we evaluate our method over such spectrum: black-box, white-box, and adaptive attacks. In particular, we use Bandits \cite{ilyas2018prior}, NES \cite{ilyas2018black} and Square \cite{andriushchenko2020square} for black-box attacks, AutoAttack \cite{autoattack} which ensembles the APGD, ADLR, FAB \cite{croce2020minimally}, and Square attacks for the white-box evaluation, and tailor an adaptive attack specific to our proposed approach for a worst-case robustness evaluation.

\textbf{Defenses Against Adversaries.}
Given the security concerns that adversarial vulnerability brings, a stream of works developed models that are not only accurate but also robust against adversarial attacks. From the black-box perspective, several defenses have shown their effectiveness in defending against such attacks \cite{PNI}. For example, injecting Gaussian noise into activation maps during both training and testing \cite{rse} was shown to successfully defend against a variety of black-box attacks \cite{dong2020benchmarking}. Moreover, SND \cite{snd} showed that small input perturbations can enhance the robustness of pretrained models against black-box attacks. However, the main drawback of randomized methods is that they can be bypassed by Expectation Over Transformation (EOT)~\cite{EOT}. Once an attacker accesses the gradients, \ie white-box attackers, the robust accuracy of such defenses drastically decreases. Thus, a stream of works built models that resist white-box attacks. While several approaches were proposed, such as regularization \cite{cisse2017parseval} and distillation \cite{distillation_papernot}, Adversarial Training (AT) \cite{madry2017towards} remains among the most effective. Moreover, recent works showed that AT can be enhanced by combining it with pretraining \cite{pretraining}, exploiting unlabeled data \cite{carmon2019unlabeled}, or concurrently, conducting transformations at test time~\cite{perez2021enhancing}. Further improvements were obtained by introducing regularizers, such as TRADES \cite{trades} and MART \cite{mart}, or combining AT with network pruning, as in HYDRA \cite{hydra}, or weight perturbations \cite{awp}.
While these methods improve the robustness, they require expensive training and degrade clean accuracy.
In this work, we show how our proposed anti-adversary layer enhances the performance of nominally trained models against realistic black-box attacks and even outperforms the strong SND defense. We show that equipping robust models with our anti-adversary layer significantly improves their robustness against black- and white-box attacks, in addition to showing worst-case robustness improvements under adaptive attacks.

\section{Methodology}
\noindent \textbf{Motivation.} Adversarial directions are the ones that maximize a loss function in the input, \ie move an input $x$ closer to the decision boundary, resulting in reducing the prediction's confidence on the correct label. In this work, we leverage this fact by prepending a layer to a trained model to generate a new input $(x+\gamma)$, which moves $x$ away from the decision boundary, thus hindering the capacity of attackers to successfully tailor adversaries. Before detailing our approach, we start with preliminaries and notations.

\begin{figure}
    \centering
    \includegraphics[width=\linewidth]{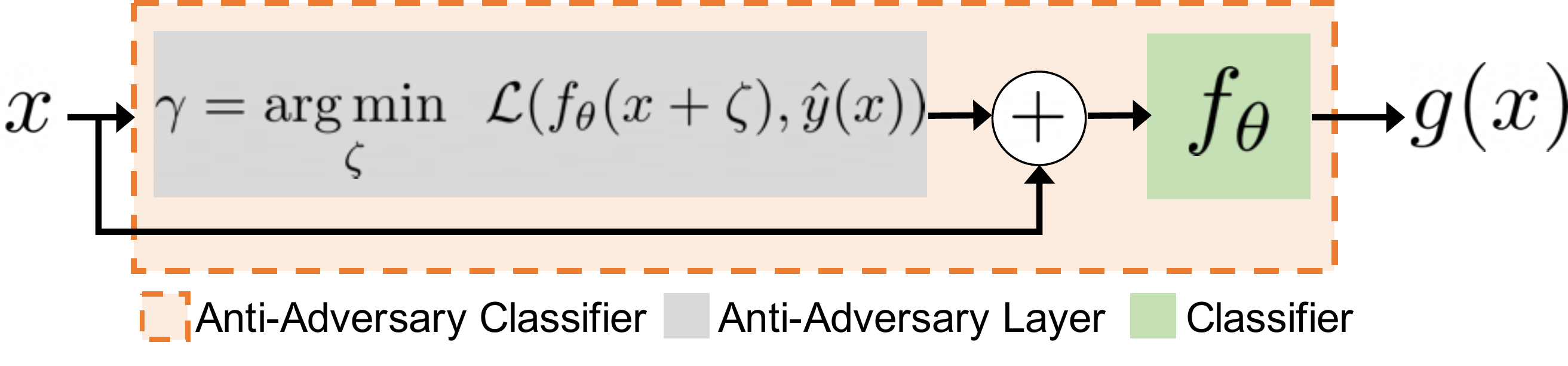}
    \caption{\textbf{The Anti-Adversary classifier.} Our anti-adversary layer generates $\gamma$ for each $x$ and $f_\theta$, and feeds $(x+\gamma)$ to $f_\theta$, resulting in our anti-adversary classifier $g$.}
    \label{fig:Algorithm_description}
\end{figure}

\subsection{Preliminaries and Notation} 
We use $f_\theta : \mathbb{R}^n \rightarrow \mathcal{P}(\mathcal{Y})$ to denote a classifier, \eg a neural network, parameterized by $\theta$, where $\mathcal{P}(\mathcal{Y})$ refers to the probability simplex over the set $\mathcal{Y} = \{1,2,\dots,k\}$ of $k$ labels. For an input $x$, an attacker constructs a small perturbation $\delta$ (\eg $\|\delta\|_p \leq \epsilon$) such that $\argmax_i f^i_\theta(x+\delta) \neq y$, where $y$ is the true label for $x$. In particular, one popular approach to constructing $\delta$ is by solving the following constrained problem with a suitable loss function $\mathcal{L}$:
\begin{equation}\label{eq:adversary}
 \max_{\delta} ~\, \mathcal{L}(f_\theta(x+\delta), y) \qquad \text{s.t.}\,\, \|\delta\|_p \leq \epsilon.
\end{equation}

Depending on the information about $f_\theta$ given to the attacker when solving Problem \eqref{eq:adversary}, the adversary $\delta$ can generally be categorized into one of three types. (\textbf{i}) \textbf{Black-box:} Only function evaluations $f_\theta$ are available when solving \eqref{eq:adversary}. (\textbf{ii}) \textbf{White-box:} Full access of the classifier $f_\theta$, e.g. $\nabla_x f_\theta$, is granted when solving \eqref{eq:adversary}. (\textbf{iii}) \textbf{Adaptive:} The attacker is tailored specifically to break the classifier $f_\theta$. That is to say, unlike white-box attacks that can be generic methods for all defenses, adaptive attacks are handcrafted to break specific defenses with full knowledge of $f_\theta$, providing a better assessment of worst-case robustness for the classifier $f_\theta$.

\subsection{Anti-Adversary Layer}
Analogous to the procedure used for constructing an adversary by solving \eqref{eq:adversary}, we propose, given a classifier, to prepend a layer that perturbs its input so as to maximize the classifier's prediction confidence at this input, hence the term \textit{anti-adversary}. Formally, given a classifier $f_\theta$, our proposed anti-adversary classifier $g$ (prepending $f_\theta$ with an anti-adversary layer) is given as follows:

\begin{equation}
\begin{aligned} \label{eq:anti_adv_layer}
    & g(x) = f_\theta(x+\gamma), \\
    & \text{s.t.} ~~ \gamma = \underset{\zeta}{\argmin} ~~\mathcal{L}(f_\theta(x+\zeta), \hat{y}(x)),
\end{aligned}
\end{equation}
\noindent where $\hat{y}(x) = \argmax_i f^i_\theta(x)$ is the predicted label. Note that our proposed anti-adversary classifier $g$ is agnostic to the choice of $f_\theta$. Moreover, it does not require retraining $f_\theta$, unlike previous works \cite{xie2018mitigating, snd} that add random perturbations to the input, further hurting clean accuracy. This is because instances that are correctly classified by $f_\theta$, \ie instances where $y = \argmax_i f^i_\theta(x)$, will be (by construction as per optimization \eqref{eq:anti_adv_layer}) classified correctly by $g$. As such, our anti-adversary layer only increases the confidence of the top prediction of $f_\theta(x)$. Also, we observe that our novel layer aligns with the recent advances in deep declarative models \cite{declarative, optnet, neural_ode, bibi2019deep}, where the output activations of a given layer are solutions to optimization problems or differential equations. We illustrate our approach in Figure \ref{fig:Algorithm_description}.

\subsection{Theoretical Motivation for Robustness}\label{sec:theory_moti}

Since the anti-adversary classifier $g$ perturbs inputs towards locations far from decision boundaries, we argue that $g$ can theoretically enjoy better robustness than $f_\theta$. In particular, we study robustness under the realistic black-box adversary setting of solving the unconstrained version of Problem \eqref{eq:adversary}. We analyze the robustness of both $g$ and $f_\theta$ under the celebrated SimBA attack \cite{simba} due to its simplicity and popularity. We show that SimBA requires a larger number of queries (forward passes) to fool $g$ than to fool  $f_\theta$, \ie $g$ is more robust than $f_\theta$. 
First, we show an equivalency between SimBA and Stochastic Three Points (STP) \cite{stp}, a recently proposed derivative-free optimization algorithm. All proofs are left for the \textbf{Appendix}.

\begin{proposition}\label{prop:simba-stp}
Let SimBA \cite{simba} with a budget of $2B$ queries select a random direction $q \in Q$, with replacement, thus updating the iterates $x^{k+1} \leftarrow x^k$ by selecting the direction among $\{\epsilon q, -\epsilon q\}$\footnote{Dropping the conditional break for loop, which is originally introduced in SimBA for computational reasons, in Algorithm (1) in \cite{simba} and evaluating on both $\pm \epsilon q$.} with the maximum $\mathcal{L}$. Then, SimBA is equivalent to STP with $B$ iterations.
\end{proposition}

Therefore, when $\mathcal{L}$ is $L$-smooth, \ie $\|\nabla_x \mathcal{L}(f_\theta(x+\delta),y) - \nabla_x \mathcal{L}(f_\theta(x),y)\| \leq L \|\delta\|$, we can find a lower bound for the number of queries $B$ required by SimBA to maximize $\mathcal{L}$ to a certain precision.

\begin{corollary}\label{cor:stp-convergence} Let $\mathcal{L}$ be $L$-smooth, bounded above by $\mathcal L(f_\theta(x^*),y)$, and the steps of SimBA satisfy $ 0 < \epsilon < \nicefrac{\rho}{nL}$ while sampling directions from the Cartesian canonical basis ($Q$ is an identity matrix here). Then, so long as:
\begin{align*}
    B > \frac{\mathcal L(f_\theta(x^*),y)-\mathcal{L}(f_\theta(x^0),y)}{(\frac{\rho}{n} - \frac{L}{2}\epsilon)\epsilon} = K_{f_\theta},
\end{align*}
we have that ~$\min_{k=1,2,\dots,B} ~\mathbb E\left[\|\nabla \mathcal L (f_\theta(x^{k}),y)\|_1\right] < \rho$.
\end{corollary}

Corollary \ref{cor:stp-convergence} quantifies the minimum query budget $K_{f_{\theta}}$ required by SimBA to maximize $\mathcal{L}(f_\theta(x),y)$, reaching a specific solution precision $\rho$ measured in gradient norm. Note that SimBA requires $2$ queries (evaluating $f_\theta$ at $x^k\pm \epsilon q$) before sampling a new direction $q$ from $Q$; thus, with a budget of $2B$, SimBA performs a total of $B$ new updates to $x^k$ with iterates ranging from $k=1$ to $k=B$. To compare the robustness of $f_\theta$ to our anti-adversary classifier $g$ described in Eq. \eqref{eq:anti_adv_layer}, we derive $K_g$, \ie the minimum query budget necessary for SimBA to achieve a similar gradient norm precision $\rho$ when maximizing $\mathcal{L}(g(x),y)$. For ease of purposes, we analyze $K_g$ when the anti-adversary layer in $g$ solves the minimization Problem \eqref{eq:anti_adv_layer} with one iteration of STP with learning rate $\epsilon_g$. Next, we show that SimBA requires a larger query budget to maximize $\mathcal{L}(g(x),y)$ as opposed to $\mathcal{L}(f_\theta(x),y)$, hence implying that $g$ enjoys improved robustness.

\begin{theorem}\label{theo:improved-robustness-ratio}
Let the assumptions in Proposition \ref{prop:simba-stp} and Corollary \ref{cor:stp-convergence} hold. Then, the anti-adversary classifier $g$ described in Eq. \eqref{eq:anti_adv_layer}, where $\gamma$ is computed with a single STP update in the same direction $q$ as SimBA but with a learning rate $\epsilon_g = (1-c) \epsilon$ with $c<1$, is more robust against SimBA attacks than $f_\theta$. In particular, $\forall ~ c
\leq 0$, SimBA fails to construct adversaries for $g$ (\ie $K_g = \infty$). Moreover, for $c \in (0,1)$, the improved robustness factor for $g$ is:

\begin{equation}
    G(c) := \frac{K_g}{K_{f_\theta}} =  \frac{\frac{\rho}{n}-\frac{L\epsilon}{2}}{(\frac{\rho}{n} - \frac{L\epsilon}{2}c)c} > 1.
\end{equation}
\end{theorem}

Theorem \ref{theo:improved-robustness-ratio} demonstrates that for any choice of $c < 1$, and under certain assumptions, $g$ is more robust than $f_\theta$ under SimBA attacks. In the case where the anti-adversary layer employs a larger learning rate $\epsilon_g$ than that of SimBA ($\epsilon$), \ie $c \leq 0$, then SimBA attacks will never alter the prediction of $g$, since $K_g = \infty$. On the other hand, when $\epsilon_g$ (the learning rate of the anti-adversary) is smaller than the learning rate of SimBA, \ie $c \in (0,1)$, SimBA will be successful in altering the prediction of $g$ but with a larger number of queries compared to $f_\theta$, that is, $g$ is more robust than $f_\theta$ under SimBA attacks. This outcome is captured by the improved robustness factor $G$, which is a strictly decreasing function in $c \in (0,1)$ and lower bounded by $1$.

\begin{algorithm}[t]
  \DontPrintSemicolon
  \SetKwFunction{FMain}{AntiAdversaryForward}
  \SetKwProg{Fn}{Function}{:}{}
  \Fn{\FMain{$f_\theta$, $x$, $\alpha$, $K$}}{
  \textbf{Initialize:} $\gamma^0 = 0$ \\
    $\hat{y}(x) =\argmax_i f^i_\theta(x)$ \\
     \For{$k = 0 \dots K-1$ }{
        $\gamma^{k+1} = \gamma^k - \alpha\, \text{sign}(\nabla_{\gamma^k} \mathcal{L}(f_\theta(x+\gamma^k), \hat{y}))$
     }
        \KwRet $f_\theta(x+\gamma^K)$ \; 
  }
   \caption{Anti-adversary classifier $g$} \label{alg:anti_adv}
\end{algorithm}

In general, we hypothesize that the stronger the anti-adversary layer solver for Problem \eqref{eq:anti_adv_layer} is, the more robust $g$ is against all attacks (including white-box and particularly against black-box attacks)\footnote{We leave to the \textbf{Appendix} a version of Theorem \ref{theo:improved-robustness-ratio}, where we derive the improved robustness factor under the white-box setting with the anti-adversary layer solving Eq. \eqref{eq:anti_adv_layer} using gradient descent.}. To that end, and throughout the paper, the anti-adversary layer solves Problem \eqref{eq:anti_adv_layer} with $K$ signed gradient descent iterations, zero initialization, and  $\mathcal{L}$ being the cross-entropy loss. Algorithm \ref{alg:anti_adv} summarizes the forward pass of $g$. Next, we empirically validate improvements in robustness over the full spectrum of adversaries.

\section{Experiments}

Evaluating robustness is an elusive problem, as it is ill-defined without establishing the information available to the attacker \eqref{eq:adversary} for constructing the adversary $\delta$. Prior works usually evaluate robustness under the adaptive, black-box or white-box settings. Here, we argue that robustness should be evaluated over the \textit{complete} spectrum of adversaries. In particular, we underscore that, while adaptive attacks can provide a worst-case robustness assessment, such assessment may be uninteresting for real deployments. For example, when the worst-case robustness of classifiers results in a draw, this tie can be broken by considering their robustness in the black-box setting, as this property increases its desirability for real-world deployment. 

Thus, we validate the effectiveness of our proposed anti-adversary classifier $g$ by evaluating robustness under adversaries from the full spectrum. (\textbf{i}) We first compare the robustness of $f_\theta$ against our proposed anti-adversary classifier $g$ with popular black-box attacks (Bandits \cite{ilyas2018prior}, NES \cite{ilyas2018black} and Square \cite{andriushchenko2020square}). We consider both cases when $f_\theta$ is nominally and robustly trained. Not only do we observe significant robustness improvements over $f_\theta$ with virtually no drop in clean accuracy, but we also outperform recently-proposed defenses, such as SND \cite{snd}. (\textbf{ii}) We further conduct experiments in the more challenging white-box setting with AutoAttack~\cite{autoattack} (in particular against the strong attacks APGD, ADLR \cite{autoattack}, and FAB \cite{croce2020minimally}), when $f_\theta$ is trained robustly with TRADES \cite{trades}, ImageNet-Pre \cite{pretraining}, MART \cite{mart}, HYDRA \cite{hydra}, and AWP \cite{awp}. (\textbf{iii})~We analyze robustness performance under tailored adaptive attacks, demonstrating that the worst-case performance is lower bounded by the robustness of $f_\theta$. In all experiments, we do \textit{not} retrain $f_\theta$ after prepending our anti-adversary layer. We set $K=2$ and $\alpha=0.15$ whenever Algorithm \ref{alg:anti_adv} is used, unless stated otherwise. (\textbf{iv}) Finally, we ablate the effect of the learning rate $\alpha$ and the number of iterations $K$ on the robustness gains.

\begin{table*}[t]
\centering
\caption{\textbf{Robustness of nominally trained models against black-box attacks:} We present the robustness of a nominally trained model against Bandits and NES, and how robustness is enhanced when equipping the model with SND~\cite{snd} and our anti-adversary layer (Anti-Adv). We perform all attacks with both 5$k$ and 10$k$ queries. Results shown are accuracy measured in $\%$ where bold numbers correspond to best performance. Our approach outperforms SND by a significant margin across datasets, attacks, and number of queries.}
\centering
\small
\begin{tabular}{c||ccccc||ccccc}
\toprule 
\midrule
    & \multicolumn{5}{c||}{ {CIFAR10}} & \multicolumn{5}{c}{ {ImageNet}} \\
    & Clean & \multicolumn{2}{c}{\text{Bandits}} & \multicolumn{2}{c||}{\text{NES}} & Clean & \multicolumn{2}{c}{\text{Bandits}} & \multicolumn{2}{c}{\text{NES}} \\
    & & \text{5K}& \text{10K} & \text{5K}& \text{10K} && \text{5K}& \text{10K}& \text{5K}& \text{10K}\\
\midrule
\text{Nominal Training}      &93.7	&24.0	 & 17.2  &5.8& 4.8	&79.2	&65.2&58.2	&22.4&21.0 \\
\text{    + SND \cite{snd}}  &92.9	&84.5& 84.3	&30.3&25.5	&79.2	&72.8&73.2	&65.4&60.2\\
\text{    + Anti-Adv}        &93.7	&\textbf{85.5}&	\textbf{86.4}&\textbf{77.0}	&\textbf{72.7}	&79.2	&\textbf{73.6}&\textbf{74.4}	&\textbf{67.2}&\textbf{66.0} \\
\midrule
\bottomrule
\end{tabular}\label{tb:nominal}
\end{table*}

\begin{table}[t]
\centering
\caption{\textbf{Equipping robustly trained models with Anti-Adv on CIFAR10 and CIFAR100 against black-box attacks.} We report clean accuracy (\%) and robust accuracy against \textit{Bandits}, \textit{NES} and \textit{Square attack} where bold numbers correspond to largest accuracy in each experiment. Our layer provides across the board improvements on robustness against all attacks without affecting clean accuracy.}
\centering
\begin{tabular}{c||c|ccg}
\toprule 
\midrule
\text{ { CIFAR10}}& \text{Clean} & \text{Bandits} & \text{NES}& \text{Square}  \\
\midrule
\text{TRADES 
}  & 85.4 & 64.7 & 74.7 & 53.1\\
\text{ + Anti-Adv }  & 85.4 & \textbf{84.6} & \textbf{83.0} & \textbf{71.7} \\
\midrule
\text{ImageNet-Pre 
}  & 88.7 & 68.4 & 78.1 &62.4\\
\text{ + Anti-Adv }  & 88.7 & \textbf{88.1} & \textbf{86.4} & \textbf{78.5}  \\
\midrule
\text{MART 
}  & 87.6 & 72.0 & 79.5 &64.9\\
\text{ + Anti-Adv }  & 87.6 & \textbf{86.5} & \textbf{85.3 }&\textbf{78.0} \\
\midrule
\text{HYDRA 
}  & 90.1 & 69.8 & 79.2 &65.0\\
\text{ + Anti-Adv }  & 90.1 & \textbf{89.4} & \textbf{87.7} &\textbf{78.8}  \\
\midrule
\text{AWP
}  & 88.5 & 71.5 & 80.1 &66.2\\
\text{ + Anti-Adv }  & 88.5 & \textbf{87.4} & \textbf{86.9} &\textbf{80.7}  \\
\midrule
\bottomrule
\midrule
 \text{ {CIFAR100}}   & \text{Clean} & \text{Bandits} & \text{NES}& \text{Square}  \\
\midrule
\text{ImageNet-Pre 
}  & 59.0 & 40.6 & 47.7 &34.6\\
\text{ + Anti-Adv }  & 58.9 & \textbf{58.2} & \textbf{55.3} &\textbf{42.4}  \\
\midrule
\text{AWP 
}  & 59.4 & 39.8 & 47.3 &34.7\\
\text{ + Anti-Adv }  & 59.4 & \textbf{57.7} & \textbf{53.8} &\textbf{46.4}  \\
\midrule
\bottomrule
\end{tabular}
\label{tb:defended_bb_cifar10}
\end{table}

\subsection{Robustness under Black-Box Attacks}
We start by studying how prepending our proposed anti-adversary layer to a classifier $f_\theta$ can induce robustness gains against black-box attacks. This is a realistic setting as several commercially-available APIs, \eg  BigML, only allow access to model predictions, and thus, they can only be targeted with black-box adversaries.

\paragraph{\textbf{Robustness when $f_\theta$ is Nominally Trained}.} We conduct experiments with ResNet18 \cite{resnets} on CIFAR10 \cite{cifars} and ResNet50 on ImageNet \cite{deng2009imagenet}. We compare our anti-adversary classifier $g$ against $f_\theta$ in terms of clean and robust test accuracy when subjected to two black-box attacks. In particular, we use the Bandits and NES attacks with query budgets of $5k$ and $10k$, and report results in Table \ref{tb:nominal}. In addition, we compare against a recently proposed approach for robustness through input randomization (SND \cite{snd}). We set $\sigma=0.01$ for SND, as it achieves the best performance. Following common practice \cite{snd}, and due to the expensive nature of evaluating Bandits and NES, all test accuracy results in Table \ref{tb:nominal} are reported on 1000 and 500 instances of CIFAR10 and ImageNet, respectively. For this experiment, we set $\alpha = 0.01$ in Algorithm \ref{alg:anti_adv}. Note that SND, the closest work to ours, outperforms the best performing defense in the black-box settings benchmarked in~\cite{dong2020benchmarking}.

\begin{table*}[t]
\centering
\caption{\textbf{Equipping robustly trained models with Anti-Adv on CIFAR10 and CIFAR100 against white-box attacks.} We report clean accuracy (\%) and robust accuracy against \textit{APGD}, \textit{ADLR}, \textit{FAB} and \textit{AutoAttack} where bold numbers correspond to largest accuracy in each experiment. The last column summarizes the improvement on the AutoAttack benchmark. We observe strong results on all models and attacks when adding our anti-adversary layer, with improvements close to $19\%$ all around.}
\centering
\begin{tabular}{c||c|ccc|g|c}
\toprule 
\midrule
  \text{ { CIFAR10}}  & \text{Clean} & \text{APGD} & \text{ADLR}& \text{FAB}&  \text{AutoAttack} & \text{Improvement} \\
\midrule
\text{TRADES
}  &84.92&	55.31&	53.12&	53.55&		53.11& \multirow{2}{*}{18.60}\\
\text{ + Anti-Adv}  &84.88&	\textbf{77.20}&	\textbf{77.05}&	\textbf{83.38}&		\textbf{71.71}  &  \\
\midrule
\text{ImageNet-Pre
}  &87.11&	57.65&	55.32&	55.69&		55.31& \multirow{2}{*}{20.70}\\
\text{ + Anti-Adv}  &87.11&	\textbf{78.76}&	\textbf{79.02}&	\textbf{85.07}&		\textbf{76.01}  &  \\
\midrule
\text{MART
}  &87.50&	62.18&	56.80&	57.34&		56.75 & \multirow{2}{*}{20.01} \\
\text{ + Anti-Adv}  &87.50&	\textbf{81.07}&	\textbf{80.54}&	\textbf{86.52}&		\textbf{76.76} & \\
\midrule
\text{HYDRA
}  &88.98&	60.13&	57.66&	58.42&		57.64 & \multirow{2}{*}{18.75}\\
\text{ + Anti-Adv}  &88.95&	\textbf{80.37}&	\textbf{81.42}& \textbf{87.92}&		\textbf{76.39}  & \\
\midrule
\text{AWP 
}  & 88.25& 63.81&	60.53&	60.98&		60.53& \multirow{2}{*}{18.68} \\
\text{ + Anti-Adv}  & 88.25&	\textbf{80.65}&	\textbf{81.47}&	\textbf{87.06}&		\textbf{79.21} & \\
\midrule
\bottomrule
\midrule
  \text{ { CIFAR100}}  & \text{Clean} & \text{APGD} & \text{ADLR}& \text{FAB}&  \text{AutoAttack} & \text{Improvement} \\
\midrule
\text{ImageNet-Pre
}  &59.37&	33.45&	29.03&	29.34&		28.96 &\multirow{2}{*}{11.72}\\
\text{ + Anti-Adv}  &58.42&	\textbf{47.63}&	\textbf{45.29}&	\textbf{53.57}&		\textbf{40.68} &   \\
\midrule
\text{AWP 
}  &60.38&	33.56&	29.16&	29.48&		29.15& \multirow{2}{*}{10.42}\\
\text{ + Anti-Adv}  &60.38&	\textbf{44.21}&	\textbf{40.32}&	\textbf{50.76}&		\textbf{39.57}& \\
\midrule
\bottomrule
\end{tabular}\label{tb:AutoAttack_cif10}
\end{table*}

As shown in Table \ref{tb:nominal}, nominally trained models $f_\theta$ are not robust: their clean accuracies on CIFAR10 and ImageNet drop from $93.7\%$ and $79.2\%$, respectively, to $4.8\%$ and $21\%$ when under black-box attacks. Moreover, while SND improves robustness significantly over $f_\theta$, \eg to $25.5\%$ on CIFAR10 and to $60.2\%$ on ImageNet, our proposed anti-adversary consistently outperforms SND across attacks, budget queries, and datasets. For instance, under the limited 5$k$ query budget, our anti-adversary classifier outperforms SND by $1\%$ and $46.7\%$ on CIFAR10 against Bandits and NES. The robustness improvements over SND increase even when attacks have a larger budget of 10$k$: on ImageNet our anti-adversary outperforms SND by $1.2\%$ against Bandits and by $5.8\%$ against NES. Further, we note that this improvement comes at \textit{no cost} on clean accuracy. In summary, Table \ref{tb:nominal} provides strong evidence suggesting that our proposed anti-adversary classifier improves the black-box robustness of a nominally trained $f_\theta$, outperforming the recent SND. In addition, this performance improvement does not hurt clean accuracy nor requires retraining $f_\theta$.

\paragraph{\textbf{Robustness when $f_\theta$ is Robustly Trained}.} 
We have provided evidence that our anti-adversary layer can improve black-box robustness of nominally trained $f_\theta$. 
Here, we investigate whether our anti-adversary layer can also improve robustness in the more challenging setting when $f_\theta$ is already robustly trained. This is an interesting setup as $f_\theta$ could have been trained robustly against white-box attacks and then deployed in practice where only function evaluations are available to the attacker \eqref{eq:adversary}, and hence only black-box robustness is of importance. Here we show we can improve black-box robustness with our proposed anti-adversary layer over five state-of-the-art robustly trained $f_\theta$: TRADES, IN-Pret, MART, HYDRA, and AWP on the CIFAR10 and CIFAR100 datasets. Similar to the previous experimental setup, and due to computational cost, we report robust accuracy on 1000 test set instances against Bandits and NES. However, for the more computationally-efficient Square attack, we report robust accuracy on the full test set.

Table \ref{tb:defended_bb_cifar10} reports the black-box robust accuracies of robustly-trained $f_\theta$ on CIFAR10 and CIFAR100, respectively. We highlight the highest scores in bold. In line with our previous observations, prepending our anti-adversary layer to $f_\theta$ has no impact on clean accuracy. More importantly, although $f_\theta$ is robustly trained and thus already enjoys large black-box robust accuracy, our proposed anti-adversary layer can boost its robustness further by an impressive $\sim15\%$. For instance, even for the top-performing $f_\theta$ (trained with AWP with a robust accuracy of $66.2\%$ on CIFAR10), our anti-adversary layer improves robustness by $14.5\%$, reaching $80.7\%$. Similarly, for CIFAR100, the anti-adversary layer improves the worst-case black-box robustness of AWP by $11.7\%$. Overall, our anti-adversary layer consistently improves black-box robust accuracy against all attacks for all robust training methods on both CIFAR10 and~CIFAR100.

\textbf{SND + Robustly Trained $f_\theta$.}
Although SND \cite{snd} does not report performance on robustly-trained models, we experiment with equipping AWP-trained models with SND. We observe that SND significantly degrades both clean and robust accuracies of AWP: employing SND on top of AWP drops clean accuracy from $88.5\%$ to $70.0\%$, while its robust accuracy (under Square attack) drops from $66.2\%$ to $59.1\%$. These results suggest that our proposed anti-adversary layer is superior to SND.

\textbf{Other Black-Box Defenses}
We compare against Random Self-Ensemble (RSE)~\cite{rse} on CIFAR10 and find that it underperforms in comparison to our approach, both in clean accuracy, with $86.7\%$, and in robust accuracy, with $78.8\%$ and $85.5\%$ under NES and Bandits, respectively. While RSE is more robust than robustly trained models against black-box attacks, equipping such models with our anti-adversary layer outperforms RSE, as illustrated, for instance, by the HYDRA+Anti-Adv row in Table~\ref{tb:defended_bb_cifar10}.

\textbf{Section Summary.} Our proposed anti-adversary layer can improve state-of-the-art robust accuracy in the realistic black-box  setting when combined with robustly trained $f_\theta$, while coming at no cost to clean accuracy. The robust black-box accuracy improvements are consistent across classifiers $f_\theta$, both with regular or robust training.

\subsection{Robustness under White-Box Attacks}

In this setting, the attacker \eqref{eq:adversary} has complete knowledge about the classifier. This challenging setup is less realistic compared to the black-box setting. Nonetheless, it is still an interesting measure of overall robustness when more information is accessible to the attacker \eqref{eq:adversary}. Various prior works~\cite{xie2019feature, feature_scattering} report robustness performance only in this setting by reporting accuracy under PGD \cite{madry2017towards} or AutoAttack.

Similar to the previous section, we experiment on CIFAR10 and CIFAR100 and assess how prepending our anti-adversary layer to robustly trained classifiers $f_\theta$ can enhance the classifiers' robustness. We report the full test robust white-box accuracy against the gradient-based attacks from AutoAttack, \ie APGD, ADLR and FAB, and also measure the accuracy under AutoAttack, defined as the worst-case accuracy across these four attacks (three white-box attacks in addition to Square attacks) under $\epsilon=\nicefrac{8}{255}$ in \eqref{eq:adversary}. We underscore that the AutoAttack ensemble is currently the standard for benchmarking defenses, \ie it is the \textit{de facto} strongest attack in this setting.

\begin{figure}
    \centering
    \includegraphics[width=0.80\linewidth]{ 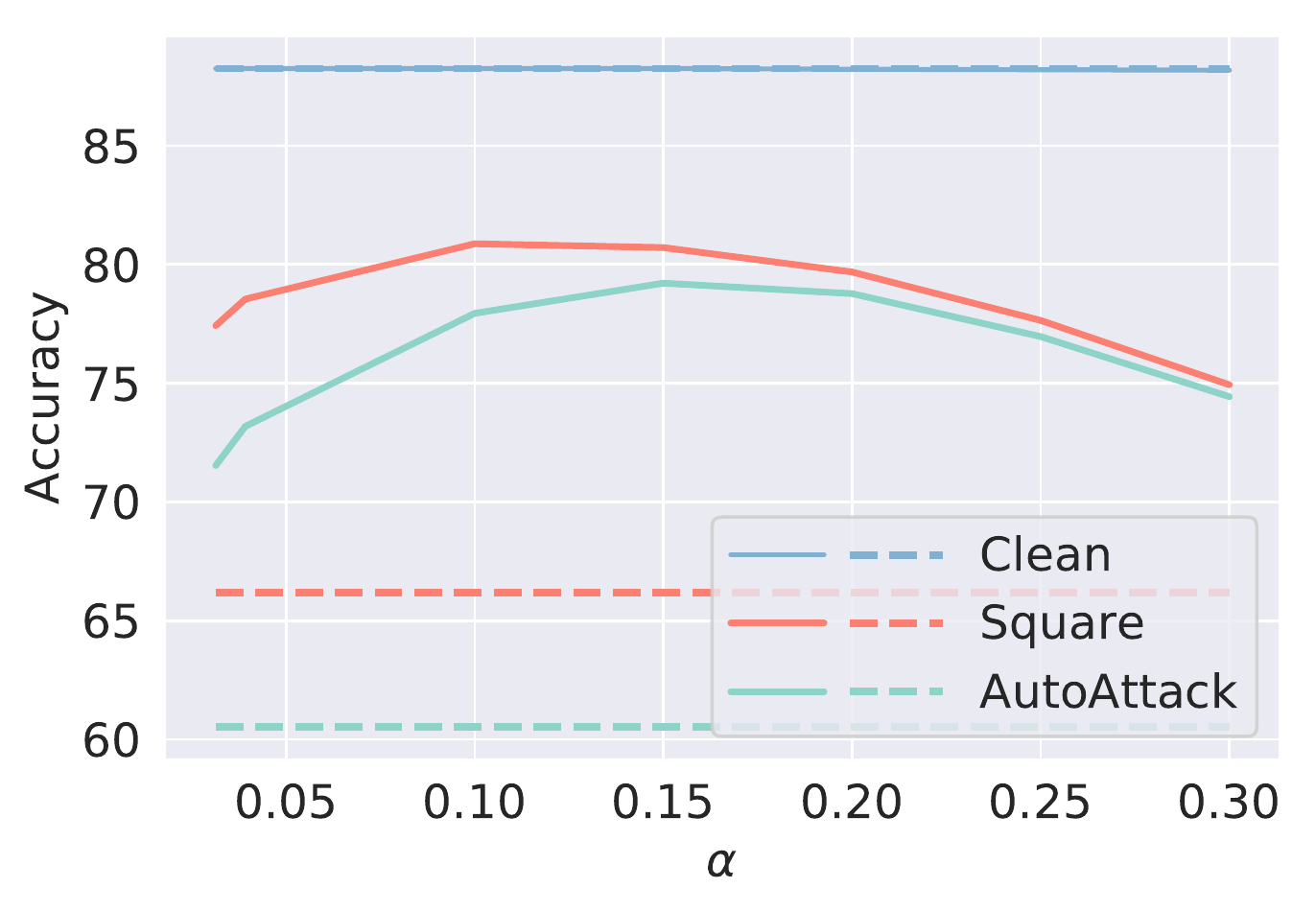}
    \caption{\textbf{Effect of varying $\boldsymbol \alpha$ on clean and robust accuracy for AWP+Anti-Adv on CIFAR10.} Dashed lines depict AWP's performance. Our layer provides substantial improvements on robust accuracy with different choices of $\alpha$ and with no effect on clean accuracy.}
    \label{fig:Ablating_learning_rate}
\end{figure}

In Table \ref{tb:AutoAttack_cif10}, we report robust accuracies on CIFAR10 and CIFAR100, respectively, and highlight the strongest performance in bold. We first observe that our anti-adversary layer improves robust accuracy by an impressive $\sim19\%$ on average against AutoAttack. In particular, for AWP, the strongest defense we consider, adversarial robustness increases from $60.53\%$ to an astounding $79.21\%$. We further observe similar results for CIFAR100: Table \ref{tb:AutoAttack_cif10} shows that the anti-adversary layer adds an average improvement of $\sim11\%$. For instance, the adversarial robustness of ImageNet-Pre increases from $28.96\%$ to over $40\%$. The improvement is consistent across all defenses on CIFAR100, with a worst-case drop in clean accuracy of $~1\%$. We also compare our approach against SND under this setup (as the experiments in SND~\cite{snd} do not study its interaction with robust training). Notably, equipping AWP with SND comes at a notable drop in clean accuracy (from $88.25\%$ to $70.03\%$) along with a drastic drop in robust accuracy (from $60.53\%$ to $27.04\%$) under AutoAttack on CIFAR10.

\noindent \textbf{Section Summary.} Our experiments suggest that, even in the challenging setting where the attacker \eqref{eq:adversary} is granted access to the gradients, our anti-adversary layer still proves to provide benefits to all defenses. For both CIFAR10 and CIFAR100, the anti-adversary layer seamlessly provides vast improvements in adversarial robustness.

\subsection{Adaptive Attacks: Worst-Case Performance}

Here, we analyze the worst-case robustness of our proposed anti-adversary classifier $g$. In particular, and under the \textit{least} realistic setting, we assume that our anti-adversary classifier $g$ is fully transparent to the attacker \eqref{eq:adversary} when tailoring an adversary. Following the recommendations in \cite{adaptive}, we explore various directions to construct an attack, such as Expectation Over Transformation (EOT) \cite{EOT,adaptive}. However, since our anti-adversary layer is deterministic, as illustrated in Algorithm \ref{alg:anti_adv}, EOT is ineffective for improving the gradient estimate. Nevertheless, we note that the anti-adversary layer depends on the pseudo-label assigned by $f_\theta$ to the original instance $x$, \ie $\hat{y}(x) = \argmax_i f_\theta^i(x)$. Therefore, an attacker with access to $g$'s internal structure can first design an adversary $\delta$ such that $\hat{y}(x+\delta) \neq y$ with $\|\delta\|_p \leq \epsilon$ following \eqref{eq:adversary}, where $y$ is $x$'s label. If $\delta$ is successfully constructed in this way, it will cause both $f_\theta$ and $g$ to produce different predictions for $x$ and $(x+\delta)$. Thus, in the least realistic adversary setting, the set of adversaries that fools $f_\theta$ fools $g$ as well. Accordingly, we argue that the worst-case robust accuracy for $g$ under adaptive attacks is lower bounded by the robust accuracy of the base classifier $f_\theta$. While, as noted in previous sections, our anti-adversary layer boosts robust accuracy over all tested datasets and classifiers $f_\theta$ (nominally or robustly trained), the worst-case robustness under the least realistic setting (adaptive attacks) is lower bounded by the robustness of $f_\theta$. This highlights our motivation that prepending our layer is of a great value to existing robust models due to its simplicity and having no cost on clean accuracy.
\begin{figure}
    \centering
    \includegraphics[width=0.80\linewidth]{ 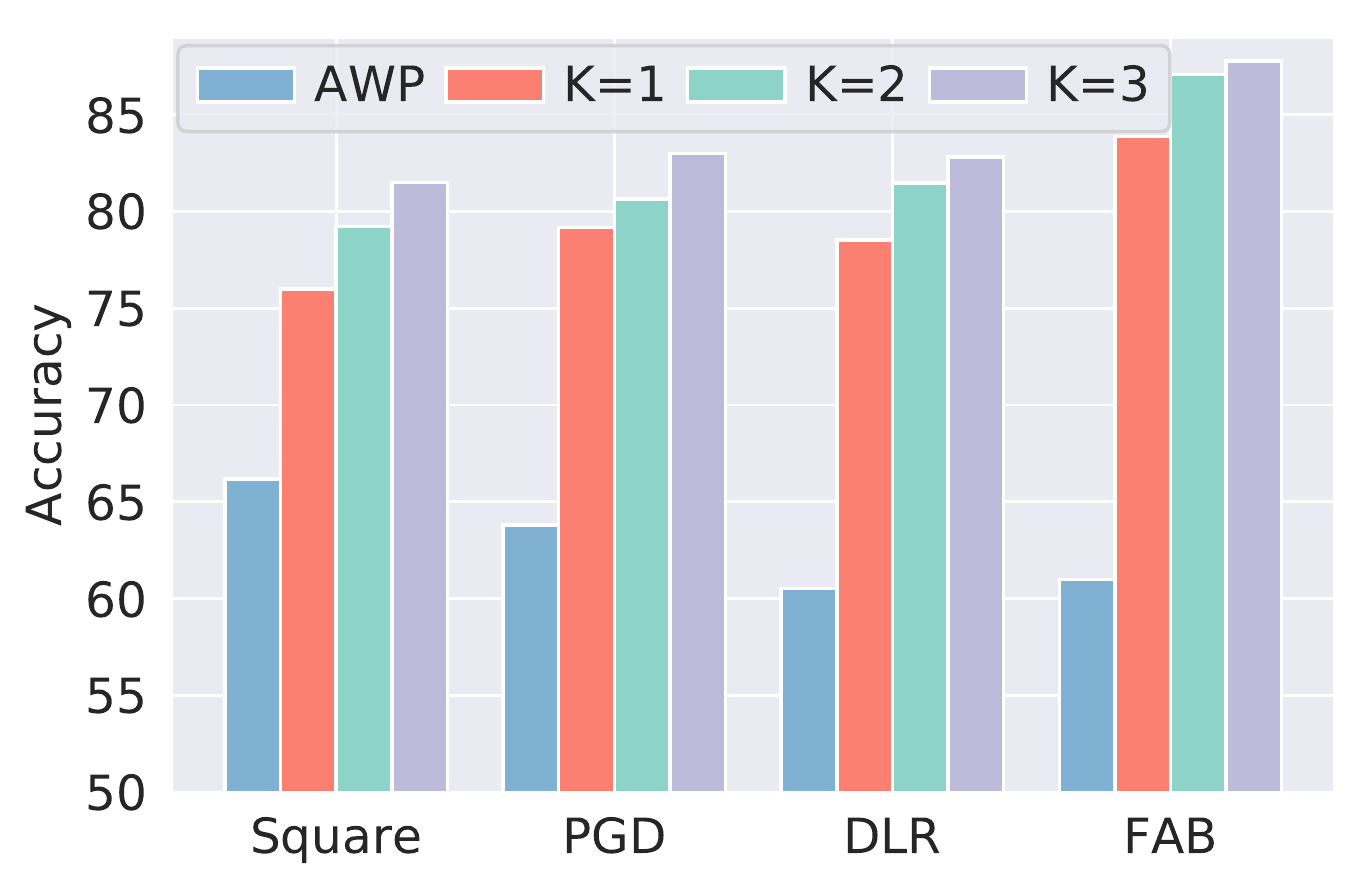}
    \caption{\textbf{Effect of varying $\boldsymbol K$ on robust accuracy for AWP+Anti-Adv on CIFAR10.} The better the solver for \eqref{eq:anti_adv_layer} is, the larger the robustness gains that our layer provides.}
    \label{fig:Ablating_K}
\end{figure}
\subsection{Ablations}
Our proposed Algorithm \ref{alg:anti_adv} has two main parameters: the learning rate, $\alpha$, and the number of iterations, $K$. We ablate both to assess their effect on robustness. All experiments are conducted on a robustly-trained $f_\theta$ (with AWP). First, we fix $K=2$ and vary $\alpha$ in the set $\{\nicefrac{8}{255}, \nicefrac{10}{255}, 0.1, 0.15, 0.2, 0.25, 0.3\}$. In Figure \ref{fig:Ablating_learning_rate}, we compare $f_\theta$ to our anti-adversary classifier $g$ in terms of clean and robust accuracies under a black-box (Square) and a white-box (AutoAttack) attacks. As shown in blue, the effect of $\alpha$ on clean accuracy is almost non-existent. On the other hand, while the robust accuracy varies with $\alpha$, the robustness gain of $g$ over $f_\theta$ is always $\ge 10\%$ for all $\alpha$ values. Next, we study the effect of varying $K \in \{1, 2, 3\}$ while fixing $\alpha=0.15$. Results in Figure \ref{fig:Ablating_K} show that all choices of $K$ lead to significant improvement in robustness against all attacks, with $K=3$ performing best. This confirms our claim that the better the solver for \eqref{eq:anti_adv_layer}, the better the robustness performance of our anti-adversary classifier. Note that while one could further improve the robustness gains by increasing $K$, this improvement comes at the expense of more computations. It is worthwhile to mention that the cost of computing the anti-adversary is $(K+1)$ forward and $K$ backward passes, which is marginal for small values of $K$. Finally, we leave more ablations, the implementation details, and the rest of our experimental results to the \textbf{Appendix}.

\section{Conclusion}
We present the anti-adversary layer, a novel training-free and theoretically supported defense against adversarial attacks. Our layer provides significant improvements in network robustness against black- and white-box attacks.

\textbf{Acknowledgement.} This publication is based upon work supported by the King Abdullah University of Science and Technology (KAUST) Office of Sponsored Research (OSR) under Award No. OSR-CRG2019-4033. We would also like to thank Humam Alwassel for the help and discussion.
\bibliography{references}

\newpage
\onecolumn
\appendix

\section{Proofs}
\begin{prop}\label{supp:prop_1}
Let SimBA \cite{simba} with a budget of $2B$ queries select a random direction $q \in Q$, with replacement, thus updating the iterates $x^{k+1} \leftarrow x^k$ by selecting the direction among $\{\epsilon q, -\epsilon q\}$ with the maximum $\mathcal{L}$. Then, SimBA is equivalent to STP with $B$ iterations.
\end{prop}
\begin{proof} 
As per Proposition \ref{supp:prop_1} and Algorithm 1 in \cite{simba}, SimBA updates can be written as follows:
\begin{align*}
    &\text{Sample a direction $q$ from $Q$ with replacement}\\
    & \text{Generate $x_+ = x^k + \epsilon q$ and $x_- = x^k - \epsilon q$}\\
    & \text{Update: } x^{k+1} = \argmax\{\mathcal L (f_\theta(x_+), y), \mathcal L (f_\theta(x_-), y), \mathcal L (f_\theta(x^k), y)\}
\end{align*}
Thus, each iteration is updating the adversarial example with one STP \cite{stp} step. Note that each iteration requires two queries ($f_\theta(x_+)$ and $f_\theta(x_-)$) since $f_\theta(x^k)$ is pre-computed in the previous iteration concluding the proof. 
\end{proof}

\begin{cor} Let $\mathcal{L}$ be $L$-smooth, bounded above by $\mathcal L(f_\theta(x^*),y)$, and the steps of SimBA satisfy $ 0 < \epsilon < \nicefrac{\rho}{nL}$ while sampling directions from the Cartesian canonical basis ($Q$ is an identity matrix here). Then, so long as:
\begin{align*}
    B > \frac{\mathcal L(f_\theta(x^*),y)-\mathcal{L}(f_\theta(x^0),y)}{(\frac{\rho}{n} - \frac{L}{2}\epsilon)\epsilon} = K_{f_\theta},
\end{align*}
we have that ~$\min_{k=1,2,\dots,B} ~\mathbb E\left[\|\nabla \mathcal L (f_\theta(x^{k}),y)\|_1\right] < \rho$.
\end{cor}
\begin{proof}
This follows directly by combining proposition \ref{prop:simba-stp} and Theorem 4.2 in~\cite{stp}, and observing that $x > \ceil{x} - 1$.
\end{proof}

\begin{definition}(Local Monotonicity) A function $f :\mathbb{R}^n \rightarrow \mathbb{R}$ is locally monotone over a set $\mathcal R$ and a scalar $\beta$ when for any:
\begin{equation}
\begin{aligned}
     f(x) \leq f(x+\alpha_i r) \iff 
    f(x) \geq f(x-\alpha_j r) \quad \forall r \in \mathcal R, \,\, \alpha_i < \beta,\,\, \alpha_j < \beta 
\end{aligned}
\end{equation}

\end{definition}

\begin{theo}\label{supp:theo1}
Let the assumptions in Proposition \ref{prop:simba-stp} and Corollary \ref{cor:stp-convergence} hold and let $\mathcal L$ be locally monotone over $Q$ and $\nicefrac{\rho}{nL}$. Then, the anti-adversary classifier $g$ described in Eq. \eqref{eq:anti_adv_layer}, where $\gamma$ is computed with a single STP update in the same direction $q$ as SimBA but with a learning rate $\epsilon_g = (1-c) \epsilon$ with $c<1$, is more robust against SimBA attacks than $f_\theta$. In particular, $\forall ~ c
\leq 0$, SimBA fails to construct adversaries for $g$ (\ie $K_g = \infty$). Moreover, for $c \in (0,1)$, the improved robustness factor for $g$ is

\begin{equation}
    G(c) := \frac{K_g}{K_{f_\theta}} =  \frac{\frac{\rho}{n}-\frac{L\epsilon}{2}}{(\frac{\rho}{n} - \frac{L\epsilon}{2}c)c} > 1.
\end{equation}
\end{theo}

\begin{proof} 
Assume that the network correctly classifies the input $x$ $(\ie \hat y = y)$ for ease. The anti-adversary update is as per Theorem \ref{supp:theo1} follows a single STP update with $\epsilon_g$ as a step size with the following output:
\begin{equation}
\label{supp:anti_adv_eq}
\begin{aligned}
\gamma^k = \text{arg}\min\{\mathcal L(f_\theta(x_+), y), \mathcal L(f_\theta(x_-), y), \mathcal L(f_\theta(x), y) \} - x^k,
\end{aligned}
\end{equation}
where $x_+ = x^k + \epsilon_g q_k$ and $x_- = x^k - \epsilon_g q_k$. Similarly, the adversary single SimBA update takes the following form:
\begin{equation}
\label{supp:adv_eq}
\begin{aligned}
\delta^k = \text{arg}\max\{\mathcal L(f_\theta(x_+'), y), \mathcal L(f_\theta(x_-'), y)), \mathcal L(f_\theta(x), y) \} - x^k,
\end{aligned}
\end{equation}
where $x_+ = x^k + \epsilon q_k$ and $x_- = x^k - \epsilon q_k$. Thus, an iteration of SimBA adversary for our classifier $g$ that has the anti-adversary layer (substituting Equations \ref{supp:anti_adv_eq} and \ref{supp:adv_eq}) takes the following form:
\begin{align*}
    x^{k+1} &= x^k + \gamma^k + \delta^k \\
    &= x^k + \text{arg}\min\{\mathcal L(f_\theta(x_+), \hat y), \mathcal L(f_\theta(x_-), \hat y), \mathcal L(f_\theta(x),\hat y) \} - x^k \\
    & \quad \quad \, \,+ \text{arg}\max\{\mathcal L(f_\theta(x_+'), \hat y), \mathcal L(f_\theta(x_-'), \hat y), \mathcal L(f_\theta(x),\hat y) \} - x^k.
\end{align*}
Without loss of generality, and due to local monotonicity of $\mathcal{L}$, when $x_+ = \text{arg}\min\{.\}$ then $x_-' = \text{arg}\max\{.\}$. Thus, we have that the adversarial update is given as follows:

\begin{equation}\label{iterate}
    x^{k+1} = -x^k + x_+' + x_- = x^k + (\epsilon - \epsilon_g) q_k.
\end{equation}

\noindent (\textbf{i}) If $\epsilon_g = \epsilon$, then $x^{k+1} = x^k$, and thus the attacker does not succeed in constructing the adversary $(\eg K_g = \infty)$. (\textbf{ii}) If $\epsilon_g > \epsilon$, then $(\epsilon - \epsilon_g)q_k$ is in the direction of $x_-$, hence $\mathcal L(x^{k+1},\hat y) < \mathcal L (x^k, \hat y)$. In this setting, the attacker fails at constructing an adversary as the loss is being decreasing. (\textbf{iii}) Lastly, if $\epsilon > \epsilon_g$, then the iterate in \eqref{iterate} is maximizes $\mathcal{L}$ with an effective learning rate of $(\epsilon-\epsilon_g = c\epsilon)$. Therefore, based on Theorem 4.2 in \cite{stp}, to guarantee $\rho-$convergence, we need
\[
B > \frac{\mathcal L(f_\theta(x^*),y)-\mathcal{L}(f_\theta(x^0),y)}{(\frac{\rho}{n} - \frac{L\epsilon}{2} c)\epsilon c} = K_{g}.
\]
The improved robustness factor, excess query budget to maximize $\mathcal{L}$ for $g$ as compared to $f_\theta$, is $\nicefrac{K_g}{K_{f_\theta}}$. Note that $G(c) > 1$ since $G$ is monotonically decreasing,
\begin{align*}
    \frac{dG(c)}{dc} = \frac{-(\frac{\rho}{n} - \frac{L\epsilon}{2})(\frac{\rho}{n} - L\epsilon c)}{((\frac{\rho}{n} - \frac{L\epsilon}{2} c)\epsilon c)^2} < 0,
\end{align*}
for $\epsilon < \frac{\rho}{nL}$ and that $\underset{c\rightarrow 1}{\lim} G(c) = 1$, concluding the proof.
\end{proof}

\section{Extending Theorem \ref{theo:improved-robustness-ratio} for White-Box Adversary}
We analyze the robustness of our proposed anti-adversary classifier $g$ in a similar spirit to Theorem \ref{theo:improved-robustness-ratio} but in the white-box setting where the attacker is granted gradient access solving \eqref{eq:adversary}. For ease of exposition, in this scenario, the anti-adversary is assumed to be computed with a single gradient descent iteration $K=1$. Below, we restate the classical result on the convergence rate of gradient descent on smooth non-convex loss $\mathcal{L}$ with a classifier $f_\theta$ \cite{schmidt_slides}.

\begin{theo}\label{prop:gd_attacker}
Let $\mathcal L$ be L-smooth and bounded above by  $\mathcal{L}(f_\theta(x^*), y)$. Let the attacker solve \eqref{eq:adversary} with $B$ iterations of gradient ascent, \ie $x^{k+1} = x^k + \frac{\alpha}{L}\nabla_{x} \mathcal{L}(f_\theta(x^k), y)$ with $\alpha \in (0,1]$. If 
\[ B > \frac{2L(\mathcal L (f_\theta(x^*), y) - \mathcal{L}(f_\theta(x^0), y)}{(2-\alpha)\alpha \rho} = K_{f_\theta},
\]
then $\min_{k=1,2,\dots,B} ~\mathbb \|\nabla \mathcal L (f_\theta(x^{k}),y)\|_2^2 < \rho$.
\end{theo}

Now, we develop a similar result to the one in Theorem \ref{theo:improved-robustness-ratio} for when the anti-adversary is solving \eqref{eq:anti_adv_layer} by a single gradient descent step.

\begin{theo}\label{theo:gd_antiadversary}
Let the assumptions of Theorem \ref{prop:gd_attacker} hold. Then, the anti-adversary classifier $g$ described in Equation \eqref{eq:anti_adv_layer}, where $\gamma$ is computed by a single gradient descent step with a learning rate $\nicefrac{\alpha_g}{L}$ where $\alpha_g = (1-c)\alpha$  with $c<1$, is more robust against the gradient ascent attack in Theorem \ref{prop:gd_attacker}. In particular, $\forall ~ c \leq 0$, the attacker fails to construct adversaries for $g$ (\ie $K_g = \infty$). Moreover, for $c \in (0,1)$, the improved robustness factor for $g$ is

\begin{equation}
    G(c) = \frac{K_g}{K_{f_\theta}} =  \frac{2-\alpha}{(2-\alpha c)c} > 1.
\end{equation}
\end{theo}

Similarly to Theorem \ref{theo:improved-robustness-ratio}, $G$ is a strictly decreasing function that is lower bounded by 1. Also, the attacker will not be able to construct an adversary when $c \leq 0$ since $K_g = \infty$. On the other hand, if $c \in (0,1)$, then $g$ is more robust than $f_\theta$, and the improvement is captured in $G$.

\begin{proof}
As per Theorem \ref{theo:gd_antiadversary}, an adversary update has the following form:

\begin{align*}
    x^{k+1} &= x^{k} + \frac{\alpha}{L}\nabla_{x} \mathcal{L}(f_\theta(x^k), y) - \frac{\alpha_g}{L}\nabla_{x} \mathcal{L}(f_\theta(x^k), y) \\
    &= x^{k} + \frac{(\alpha - \alpha_g)}{L} \nabla_{x} \mathcal{L}(f_\theta(x^k), y)  \\ 
    &= x^{k} + \frac{c\alpha}{L} \nabla_{x} \mathcal{L}(f_\theta(x^k), y) 
\end{align*}
    Thus, if $c \leq 0$, then $\mathcal{L}(f_\theta(x^{k+1}), y) \leq \mathcal{L}(f_\theta(x^k), y)$. That is to say that, the attacker in each iteration is not maximizing  $\mathcal{L}(f_\theta(x), y)$, and hence fails at constructing an adversary $(\eg K_g = \infty)$. On the other hand, if $c\in(0,1)$, then the previous iteration is one gradient ascent step to maximize $\mathcal{L}(f_\theta(x^k), y)$ with a learning rate of $\nicefrac{c\alpha}{L}$. Therefore, based on Theorem \ref{prop:gd_attacker}, to guarantee that the adversary successfully attacks $g$ for with $\rho$ gradient precision, we need:
    \[
    B > \frac{2L(\mathcal L (f_\theta(x^*), y) - \mathcal{L}(f_\theta(x^0), y)}{(2-\alpha c)\alpha \rho c} = K_{g}.
    \]
    Note that to get the improved robustness factor $G$ we divide $\nicefrac{K_g}{K_{f_\theta}}$ observing that $G(c) > 1$ and is monotonically decreasing since
    \[\frac{dG(c)}{dc} = \frac{-(2-\alpha)(2-2\alpha c)}{(2c-\alpha c^2)^2} < 0,
    \]
    and that both $\alpha < 1$ and $c < 1$. Moreover, $\underset{c\rightarrow 1}{\lim} G(c) = 1$ concluding the proof.
\end{proof}
\newpage
\section{Extra Experiments}
\textbf{Decision Based Black-Box Attacks (DBAs).} We conduct experiments against a weaker type of black-box attacks, namely decision based attacks. In this setup, the attacker has access to the output class of the model. We conduct experiments against two attacks: Boundary Attack (BA) \cite{brendel2018decision} and QEBA \cite{cheng2018queryefficient} attack with $\epsilon = \nicefrac{8}{255}$. Due to the large number of queries that these attacks require, we were able to report the results on a 100 randomly selected samples from CIFAR10 on Table \ref{tb:decision-based}. We observe that our anti-adversary layer, when combined with AWP, provides a notable improvement against QEBA attack while not degrading the performance against BA. This extends the benefits of our anti-adversary layer to DBAs.

\begin{table}
\label{table:table1}
\centering
\caption{\textbf{Robustness against Boundary Attack:} We equip AWP with our anti-adversary layer and conduct DBAs.}
\centering
\begin{tabular}{c||c|cc}
\toprule 
\midrule
                        & Clean         & BA & QEBA               \\
\midrule
\text{AWP [41]}              & 88.0          & 86.0             & 85.0          \\
\text{ + Anti-Adv}      & 88.0          & 86.0             & \textbf{86.0}   \\
\midrule
\bottomrule
\end{tabular}\label{tb:decision-based}
\end{table}

\textbf{Experiments with $K=1$.} For completeness, we conduct both black-box (using Square attack) and white-box with the cheapest version of our Algorithm \ref{alg:anti_adv} where we set $K=1$. In this experiment, we vary the learning rate $\alpha \in \{\nicefrac{10}{255}, 0.15\}$ reporting the results on both CIFAR10 and CIFAR100. As shown in Tables \ref{tb:AutoAttack_cif10_app} and \ref{tb:AutoAttack_cif100_app}, our anti-adversary layer even with $K=1$ provides a remarkable improvement on network robustness against both black-box and white-box attacks. In particular, we improve the robust accuracy against the strongest black-box attack (shaded in grey), Square attack, by at least $5\%$ and $6\%$ on CIFAR10 and CIFAR100, respectively. This improvement extends to cover the white-box settings as well, where the improvement against AutoAttack is $13\%$ on CIFAR10 and $5\%$ on CIFAR100.

\begin{table*}[t]
\centering
\caption{\textbf{Equipping robustly trained models with Anti-Adv on CIFAR10 against black- and white-box attacks.} We report clean accuracy (\%) and robust accuracy against \textit{APGD}, \textit{ADLR}, \textit{FAB}, \textit{Square} and \textit{AutoAttack} where bold numbers correspond to highest accuracy in each experiment. The last column summarizes the improvement on the AutoAttack benchmark.}
\centering
\begin{tabular}{c||c|cccg|c|c}
\toprule 
\midrule
    & \text{Clean} & \text{APGD} & \text{ADLR}& \text{FAB}& \text{Square}& \text{AutoAttack} & \text{Improvement} \\
\midrule
\text{ImageNet-Pre}  &87.11&	57.65&	55.32&	55.69&	62.39&	55.31& - \\
\text{ \,\,  + Anti-Adv $(\alpha=\nicefrac{10}{255})$}  &87.11&	61.92&	59.06&	73.93&	69.01&	58.77  & 3.46 \\
\text{ + Anti-Adv $(\alpha=0.15)$}  &87.02&	\textbf{77.74}&	\textbf{75.09}&	\textbf{81.25}&	\textbf{75.72}&	\textbf{72.63}  & \textbf{17.32} \\
\midrule
\text{MART}  &87.50&	62.18&	56.80&	57.34&	64.87&	56.75 &-\\
\text{ \,\, + Anti-Adv $(\alpha=\nicefrac{10}{255})$}  &87.50	&70.98	&65.03	&77.15	&75.47	&64.51& 7.73 \\
\text{ + Anti-Adv $(\alpha=0.15)$}  &87.29	&\textbf{75.67}	&\textbf{72.90}	&\textbf{79.69}	&\textbf{70.00}	&\textbf{67.42}& \textbf{10.67} \\
\midrule
\text{HYDRA}  &88.98&	60.13&	57.66&	58.42&	65.01&	57.64 & -\\
\text{ \,\, + Anti-Adv $(\alpha=\nicefrac{10}{255})$}  &88.98	&71.84	&69.35	&83.72	&76.87	&68.98& 11.34 \\
\text{ + Anti-Adv $(\alpha=0.15)$}  &88.93	&\textbf{78.55}	&\textbf{78.27}	&\textbf{84.36}	&\textbf{75.98}	&\textbf{73.59}& \textbf{15.95} \\
\midrule
\text{AWP}  & 88.25& 63.81&	60.53&	60.98&	66.18&	60.53& - \\
\text{ \,\, + Anti-Adv $(\alpha=\nicefrac{10}{255})$}  &88.25	&70.86	&68.80	&82.06	&75.39	&68.57& 8.04 \\
\text{ + Anti-Adv $(\alpha=0.15)$}  &88.10	&\textbf{79.16}	&\textbf{78.52}	&\textbf{83.88}	&\textbf{76.00}	&\textbf{74.47}& \textbf{13.34} \\
\midrule
\bottomrule
\end{tabular}\label{tb:AutoAttack_cif10_app}
\end{table*}

\begin{table*}[t]
\centering
\caption{\textbf{Equipping robustly trained models with Anti-Adv on CIFAR100 against black- and white-box attacks.} Similarly to CIFAR10 experiments, our layer provides a sizable improvement to robustness without sacrificing clean accuracy.}
\centering
\begin{tabular}{c||c|cccg|c|c}
\toprule 
\midrule
    & \text{Clean} & \text{APGD} & \text{ADLR}& \text{FAB}& \text{Square}& \text{AutoAttack} & \text{Improvement} \\
\midrule
\text{ImageNet-Pre}  &59.37&	33.45&	29.03&	29.34&	34.55&	28.96 &-\\
\text{ \,\, + Anti-Adv $(\alpha=\nicefrac{10}{255})$}  &59.29	&34.91	&30.87	&39.46	&39.44	&30.61& 1.65 \\
\text{ + Anti-Adv $(\alpha=0.15)$}  &59.24	&\textbf{35.55}	&\textbf{31.56}	&\textbf{40.86}	&\textbf{40.76}	&\textbf{31.34}& \textbf{2.38} \\
\midrule
\text{AWP}  &60.38&	33.56&	29.16&	29.48&	34.66&	29.15&-\\
\text{ \,\, + Anti-Adv $(\alpha=\nicefrac{10}{255})$}  &60.38	&34.30	&30.17	&38.09	&36.88	&30.15& 1.00 \\
\text{ + Anti-Adv $(\alpha=0.15)$}  &60.38	&\textbf{39.16}	&\textbf{35.30}	&\textbf{47.18}	&\textbf{44.30}	&\textbf{34.88}& \textbf{5.73} \\
\midrule
\bottomrule
\end{tabular}\label{tb:AutoAttack_cif100_app}
\end{table*}

\textbf{Experiments with $K\in\{4,5\}$.} For completeness, we analyze the effect of enlarging the number of iterations $K$ used to solve Equation \eqref{eq:anti_adv_layer}. In Figure \ref{fig:Ablating_K_appendix}, we show the robust accuracy of AWP when combined with our anti-adversary layer when varying $K\in\{2, 3, 4, 5\}$ with $\alpha=0.15$. We observe that the larger the number of iterations used, the larger the robustness gains that our layer brings.

\begin{figure}
    \centering
    \includegraphics[width=0.5\linewidth]{ 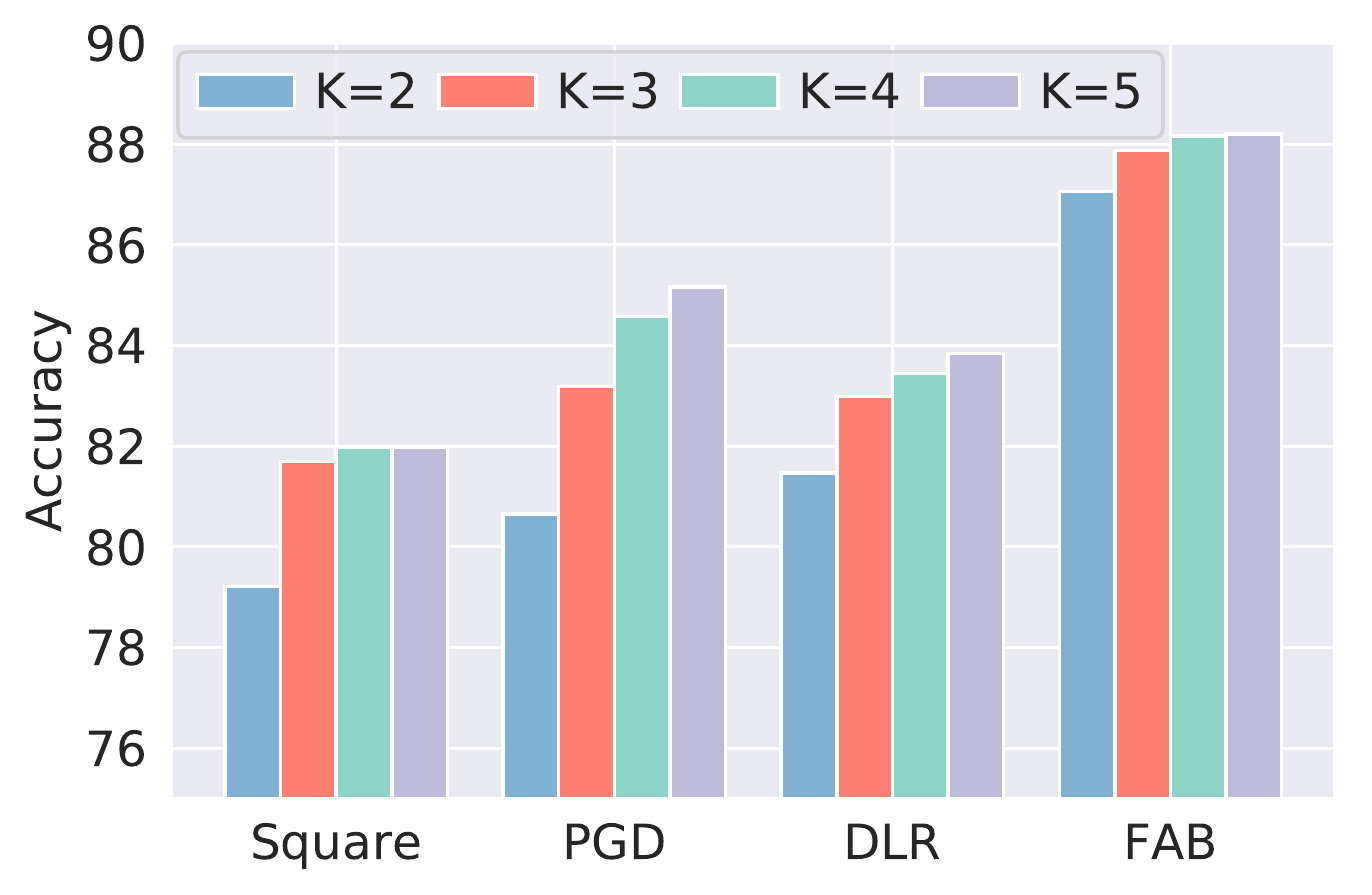}
    \caption{\textbf{Effect of varying $\boldsymbol K$ on robust accuracy for AWP+Anti-Adv on CIFAR10.} The better the solver for \eqref{eq:anti_adv_layer} is, the bigger the robustness gains that our layer provides.}
    \label{fig:Ablating_K_appendix}
\end{figure}

\section{Implementation Details}
\textbf{Nominally trained models.}
For CIFAR10 models, we trained ResNet18 from scratch for 90 epochs with SGD with an initial learning rate of 0.1, momentum of 0.9, and weight decay of $2\times 10^{-4}$. We multiply the step size by 0.1 after every 30 epochs. For ImageNet experiments, we used pretrained weights of ResNet50 from PyTorch \cite{pytorch_neurips}. For all robust models, we used provided weights by the respective authors. 

\textbf{Black-box attacks.} We used NES and Bandits from

\texttt{https://raw.githubusercontent.com/MadryLab/blackbox-bandits/master/src/main.py}, while we used the Square attack from the AutoAttack repo at \texttt{https://github.com/fra31/auto-attack}.

\textbf{White-box attacks.} We used APGD, ADLR, FAB and the worst case accuracy AutoAttack from the aforementioned AutoAttack repository.

\end{document}